\documentclass[11pt]{article}

\usepackage[final]{acl}
\usepackage{times}
\usepackage{latexsym}
\usepackage[T1]{fontenc}
\usepackage[utf8]{inputenc}
\usepackage{microtype}
\usepackage{inconsolata}
\usepackage{xspace}
\usepackage{graphicx}
\usepackage{subcaption}
\usepackage{amsmath}
\usepackage{amssymb}
\usepackage{amsthm}
\usepackage{algorithm}
\usepackage{algpseudocode}
\usepackage{booktabs}
\usepackage{colortbl}
\usepackage{tabularx}
\usepackage{tcolorbox}
\tcbuselibrary{breakable}
\definecolor{oursgray}{gray}{0.94}
\definecolor{ourstableblue}{RGB}{212,243,253}
\definecolor{stgray}{gray}{0.55}
\definecolor{pacPurple}{RGB}{116,62,173}
\definecolor{pacLearnBlue}{RGB}{0,92,175}
\definecolor{pacConvertOrange}{RGB}{183,95,0}
\newcommand{\pac}{{\textbf{\texttt{PAC}}}\xspace}
\newcommand{\stcell}[1]{\textcolor{stgray}{#1}}
\newcommand{\stddev}[1]{\nobreak\hspace{0.15em}{$\pm\,#1$}}
\newcommand{\LearnFactor}[1]{\textcolor{pacLearnBlue}{\mathbf{Learnability}_{#1}}}
\newcommand{\ConvertFactor}[1]{\textcolor{pacConvertOrange}{\mathbf{Convert}_{#1}}}
\hypersetup{
  pdftitle={PAC: Progress-Augmented Advantage Curriculum for Multi-Task Reinforcement Learning of LLMs},
  pdfauthor={Yuanqiang Yu, Yanzhao Zheng, Zhentao Zhang, Tianze Xu, Chao Ma, Jihuai Zhu, Jiashun Liu, Xinle Deng, Baohua Dong, Hangcheng Zhu, Ruohui Huang}
}
\newtcolorbox{promptbox}[1]{
  colback=gray!3!white,
  colframe=gray!45!black,
  colbacktitle=gray!12!white,
  coltitle=black,
  fonttitle=\bfseries\small,
  title={#1},
  before upper={\raggedright},
  boxrule=0.35pt,
  arc=1.5pt,
  left=1.8mm,
  right=1.8mm,
  top=1mm,
  bottom=1mm,
  before skip=0.7\baselineskip,
  after skip=0.7\baselineskip,
  breakable
}
\newtheorem{theorem}{Theorem}[section]

\title{PAC: Progress-Augmented Advantage Curriculum for Multi-Task Reinforcement Learning of LLMs}

\author{
  \textbf{Yuanqiang Yu, Yanzhao Zheng, Zhentao Zhang, Tianze Xu, Chao Ma, Jihuai Zhu}\\[-1pt]
  \textbf{Jiashun Liu, Xinle Deng, Baohua Dong\textsuperscript{*}, Hangcheng Zhu, Ruohui Huang}\\[2pt]
  Alibaba Group, Hangzhou, China\\
  \texttt{\{yuyuanqiang.yyq,zhengyanzhao.zyz\}@alibaba-inc.com}\\[-1pt]
  \texttt{\{zhangzhentao.zzt,xutianze.xtz,mc524716\}@alibaba-inc.com}\\[-1pt]
  \texttt{\{zhujihuai.zjh,baohua.dbh\}@alibaba-inc.com}\\[-1pt]
  \texttt{\{linran.lr09,wentong\}@alibaba-inc.com}
}
\begin{document}
\maketitle
\begingroup
\renewcommand{\thefootnote}{\fnsymbol{footnote}}
\footnotetext[1]{Corresponding author.}
\endgroup
\begin{abstract}
Reinforcement learning (RL) is used to improve the reasoning abilities of LLMs, while training data span heterogeneous tasks.
However, most RL post-training pipelines rely on fixed or manually designed task mixtures, even though task usefulness changes as training progresses.
Online curriculum methods often define learnability by update magnitude, ignoring whether the update translates into reward gains, which can misallocate rollout budget toward tasks with large but ineffective updates.
We propose \pac, a Progress-Augmented Advantage Curriculum for multi-task RL of LLMs that combines two task-level signals: advantage-derived learnability, which measures the magnitude of the policy update a task can induce, and recent reward gains, which show whether those updates have improved task performance.
A Bayesian Thompson Sampling controller uses these signals to allocate rollouts across tasks during GRPO training.
We evaluate \pac under two settings: a multi-level reasoning setting and a multi-domain reasoning setting.
\pac improves sample efficiency and final performance: it reaches comparable validation scores with fewer rollout steps and achieves higher final averages than random sampling and advantage-based curriculum baselines in both settings.
These results show that jointly tracking advantage signals and actual reward gains yields an effective online curriculum for LLM post-training.
\end{abstract}

\section{Introduction}

\begin{figure*}[t!]
\centering
\includegraphics[width=\textwidth]{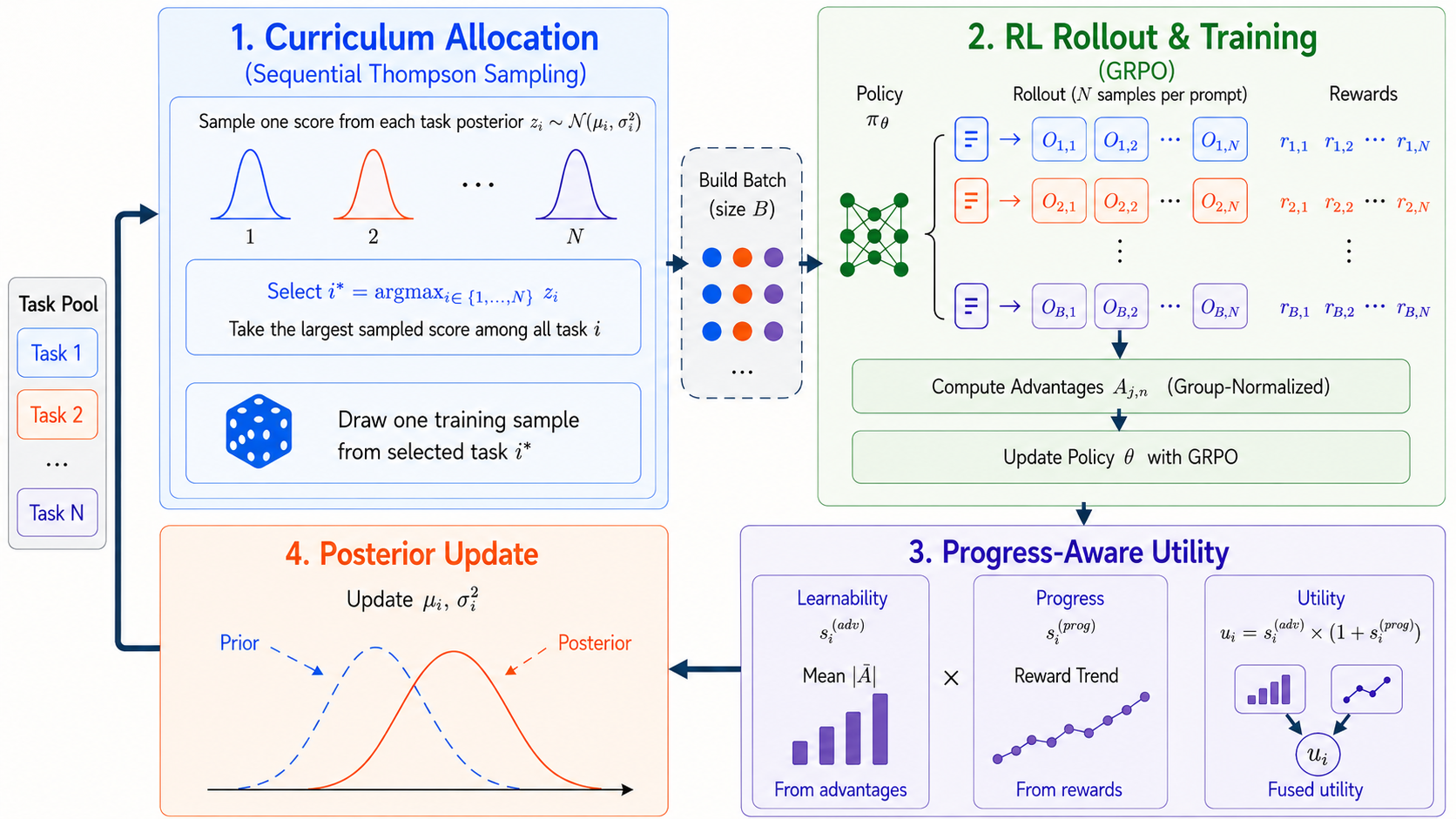}
\caption{\textbf{Illustration of \pac for multi-task RL post-training.} At each iteration, a Thompson Sampling controller allocates rollout budget across task arms to construct a mixed batch. The sampled responses are evaluated with verifiable reward functions, from which group-normalized advantages are estimated for the GRPO update. \pac fuses advantage-derived learnability with reward-derived progress into utility observations, updates task posteriors, and uses them to allocate the next rollout batch.}

\label{fig:pac_overview}
\end{figure*}

In practical LLM post-training, reinforcement learning (RL) is increasingly applied to heterogeneous task mixtures rather than to a single fixed task. Recent LLM training reports, including DeepSeek-R1 \citep{guo2025deepseekr1}, Qwen3 \citep{yang2025qwen3}, and Tulu 3 \citep{lambert2024tulu3}, describe RL post-training over mixtures of reasoning, coding, STEM, instruction following, multilingual, or multimodal data. However, these reports do not focus on adapting the task mixture online as the policy changes. This gap raises a central curriculum question: how should a trainer allocate limited rollout budget across tasks during multi-task RL post-training?

This question is especially consequential for Group Relative Policy Optimization (GRPO), a widely used RL post-training algorithm with verifiable rewards. Each update samples multiple responses to the same prompt, scores them, and derives relative advantages from the resulting rewards \citep{shao2024deepseekmath}, making rollouts a major training cost. Recent GRPO variants improve the optimizer through revised clipping, sampling, loss design, bias correction, or sequence-level policy ratios \citep{yu2025dapo,liu2025understanding,zheng2025gspo}. These optimizer-level advances are complementary to our focus: they generally assume a fixed task distribution and do not adapt the task mixture itself.

The core difficulty is that task utility is non-stationary during training. Easy tasks may stop providing useful learning signal once performance saturates; hard tasks may provide little actionable feedback early in training; and ability-dependent tasks may become useful only after their prerequisite abilities are learned. Uniform sampling ignores these dynamics, while hand-designed or difficulty-based curricula require difficulty labels, filtering rules, or fixed schedules that often transfer poorly across settings \citep{bengio2009curriculum,narvekar2020curriculum,parashar2026curriculum,shi2025efficient,song2025fastcurl}.

Recent online data-selection and curriculum methods adapt using learnability, difficulty, posterior uncertainty, gradient-based signals, or influence estimates \citep{foster2025learning,bae2025online,qu2025canprompt,shen2025bots,zeng2025cures,zhu2025dataefficient}. While these signals enable online adaptation, they do not directly measure the allocation target: whether additional rollouts on a task translate into reward improvement. A task can be uncertain, difficult, or update-inducing without producing reward gains under the current policy. In LLM post-training, recent bandit curricula often use absolute advantage as the bandit reward \citep{chen2025selfevolving,wang2025dump}. However, the absolute advantage primarily reflects the magnitude of the policy update a task can induce; it does not indicate whether the update translates into higher reward. In heterogeneous task mixtures, this distinction matters: high-variance tasks can be overprioritized even when their expected rewards are not improving, while tasks with moderate advantages but steadily improving rewards may be undersampled.

To address this limitation, we propose \pac, a Progress-Augmented Advantage Curriculum for multi-task RL post-training of LLMs. As shown in Figure~\ref{fig:pac_overview}, \pac estimates task utility from two complementary online signals: advantage-derived learnability proxies the update potential available from a task, while reward-derived progress measures whether recent updates translate into reward gains. \pac fuses the two signals into utility observations and uses a Bayesian Thompson Sampling controller to maintain task-level beliefs and allocate future rollouts under uncertainty. Our contributions are threefold. First, we highlight a limitation of advantage-only task utility for online curricula in heterogeneous LLM RL mixtures. Second, we introduce \pac, a progress-augmented Thompson Sampling curriculum for adaptive rollout allocation. Third, we validate \pac in multi-level and multi-domain reasoning settings, where it improves sample efficiency and final performance over curriculum baselines.

\section{Related Work}
\begingroup
\setlength{\parindent}{0pt}
\textbf{Curriculum Learning for SFT.}
Curriculum learning studies how ordering or weighting examples and tasks affects optimization and generalization \citep{bengio2009curriculum,narvekar2020curriculum}.
In supervised fine-tuning (SFT) of LLMs, curricula usually rely on static mixtures or precomputed properties such as task balance, data quality, difficulty, and instruction diversity \citep{sanh2022multitask,wei2022finetuned,wang2022supernaturalinstructions,wang-etal-2023-self-instruct,pmlr-v202-longpre23a,zelikman2022star,xu2023wizardlm,luo2023wizardmath}.
Because targets are fixed and task utility is largely determined before training, these methods do not address the online allocation problem in RL post-training, where task usefulness shifts as the policy evolves.

\textbf{Curriculum Learning for RL.}
RL curricula adapt tasks or environments using episodic return, learning progress, or estimated future learning potential \citep{ijcai2020p671,graves2017automated,matiisen2020teacher}.
Prioritized Level Replay revisits levels with high learning potential \citep{pmlr-v139-jiang21b}; T3S combines task-specific feature selection with a scheduler, whereas Distral focuses on knowledge transfer through a shared distilled policy \citep{yu2023t3s,teh2017distral}.
Most closely, Hard Tasks First introduces Scheduled Multi-Task Training (SMT), which prioritizes tasks according to difficulty estimates derived from returns and policy entropy and periodically resets selected network parameters \citep{pmlr-v235-cho24d}.
These RL curriculum methods are complementary to ours but do not directly target online allocation over LLM task mixtures.

\textbf{Curriculum Learning for LLM Post-Training.}
Recent RLHF and reasoning-oriented systems train on heterogeneous mixtures with PPO or GRPO, but typically rely on fixed or manually designed task mixtures rather than an explicit online allocation rule \citep{openai2024o1systemcard,guo2025deepseekr1,lambert2024tulu3}.
Existing post-training curricula either select frontier or intermediate-difficulty prompts, replay online data, or train samplers from advantage-based bandit rewards and Bayesian online task selection \citep{foster2025learning,bae2025online,sun2025improving,muhtar2026complementaryreinforcementlearning,chen2025selfevolving,shen2025bots}.
\pac is closest to advantage-based online allocation, where advantage magnitude provides a proxy for update strength but ignores whether the update yields reward improvement.
To address this gap, \pac combines advantage-derived learnability with reward-derived progress in a Thompson Sampling controller for rollout allocation.
\par\endgroup

\section{\pac: Progress-Augmented Advantage Curriculum}
\label{sec:method}

We describe \pac as an online controller for multi-task GRPO post-training.
The method treats task selection as a non-stationary allocation problem, where reward and advantage signals from recent rollouts update beliefs about which tasks are useful for the current policy.
We first define the multi-task training setting and the allocation objective.
The following sections then introduce the task utility, the Bayesian curriculum controller, and the full training algorithm.

\subsection{Problem Formulation}
\label{sec:problem_formulation}

We consider RL post-training of an LLM policy $\pi_{\theta}$ over $N$ task datasets $\mathcal{D}=\{\mathcal{D}_1,\dots,\mathcal{D}_N\}$.
At training step $t$, the curriculum controller determines the allocation of training samples over the task datasets.
We represent this allocation by a task-sampling distribution $\boldsymbol{p}^{(t)}=(p_1^{(t)},\dots,p_N^{(t)})$ over the $N$ tasks, with $p_i^{(t)}\ge 0$ and $\sum_{i=1}^{N} p_i^{(t)}=1$.
Each component $p_i^{(t)}$ denotes the probability that one training sample is assigned to task $i$; in \pac, this allocation is induced implicitly by sequential Thompson Sampling.

For a selected prompt $x$, the current policy generates $G$ rollouts $\{o_j\}_{j=1}^{G} \sim \pi_{\theta}(\cdot \mid x)$, and each rollout $o_j$ is scored by a verifiable reward $r_j$.
Under an allocation $\boldsymbol{p}^{(t)}$, a batch of size $B$ is formed by repeatedly sampling a task $i \sim \boldsymbol{p}^{(t)}$ and then sampling a prompt $x \sim \mathcal{D}_{i}$ from the selected task.
Thus, a single batch may contain training samples from several tasks.
For each rollout $o_j$, we compute the group-normalized advantage as
\[
\hat{A}_j
=
\frac{r_j - \operatorname{mean}(\{r_1,\ldots,r_G\})}
{\operatorname{std}(\{r_1,\ldots,r_G\}) + \epsilon},
\]
where the mean and standard deviation are computed over the $G$ rewards sampled for the same prompt, and $\epsilon > 0$ is a small stability constant.
Following standard policy-gradient methods, the update term for a prompt $x$ is
\begin{equation}
\begin{aligned}
g_x(\theta)
&=
\mathbb{E}_{o_{1:G}\sim \pi_{\theta}(\cdot \mid x)} \\
&\quad\!\left[
\frac{1}{G}\sum_{j=1}^{G}\hat{A}_j\,\nabla_{\theta}\log \pi_{\theta}(o_j \mid x)
\right].
\end{aligned}
\label{eq:prompt_update}
\end{equation}
Averaging $g_x(\theta)$ over prompts in task $i$ yields the task-conditioned gradient $g_i(\theta)=\mathbb{E}_{x\sim \mathcal{D}_i}[g_x(\theta)]$, and the mixed-batch GRPO gradient is
\begin{equation}
\nabla_{\theta}\mathcal{J}_{t}(\theta)
\;=\;
\sum_{i=1}^{N} p_i^{(t)} g_i(\theta).
\label{eq:multi_task_grpo}
\end{equation}
The goal is to update $\theta$ while adapting the task allocation, thereby spending the limited rollout budget on tasks that are most useful for the current policy.

We model this allocation problem as a non-stationary multi-armed bandit.
Each task is an arm, and pulling arm $i$ assigns one training sample to $\mathcal{D}_i$.
Arm utilities change as the policy is updated: a task may become learnable, saturate, or remain too difficult for the current policy.
\pac uses this feedback loop to adapt the curriculum.
GRPO rollouts produce rewards and advantages; these signals are summarized as task-level utilities, and the Bayesian controller updates its posterior before forming the next rollout batch.

\subsection{Online Task Utility Estimation}
\label{sec:reward_construction}

The ideal curriculum would allocate the next training sample to the task that yields the largest short-horizon improvement.
This marginal gain cannot be measured directly online, because a batch reveals feedback only for the tasks it contains and does not expose cross-task transfer.
\pac therefore approximates the latent marginal value of each arm using observable proxies; transfer effects are reflected later in the reward histories of affected tasks.

Let $\mathcal{V}_i(\theta)$ be the expected reward of task $i$ under $\pi_{\theta}$, and let $\Delta \theta_i^{(t)}$ be the marginal GRPO update induced by assigning one training sample to task $i$ at step $t$.
We define arm utility as the expected short-horizon gain in the task-specific reward objective.
A first-order Taylor expansion around $\theta_t$ gives
\begin{equation}
\begin{aligned}
u_i^{\star (t)}
&=
\mathbb{E}\!\left[
\mathcal{V}_i(\theta_t + \Delta \theta_i^{(t)})
-
\mathcal{V}_i(\theta_t)
\,\middle|\, x\sim\mathcal{D}_i
\right] \\
&\approx
\mathbb{E}\!\Big[
\nabla_{\theta}\mathcal{V}_i(\theta_t)^{\top}
\Delta \theta_i^{(t)}
\,\Big|\, x\sim\mathcal{D}_i
\Big].
\end{aligned}
\label{eq:latent_utility}
\end{equation}
Decomposing $\Delta \theta_i^{(t)} = \|\Delta \theta_i^{(t)}\|_2 \hat{\Delta \theta}_i^{(t)}$ into magnitude and direction separates the inner product into two interpretable factors.
The directional term $\nabla_{\theta}\mathcal{V}_i(\theta_t)^{\top}\hat{\Delta \theta}_i^{(t)}$ measures reward improvement per unit update and motivates the conversion factor $\ConvertFactor{i}$.
The magnitude term $\|\Delta \theta_i^{(t)}\|_2$ captures the size of the task-induced update and motivates the learnability factor $\LearnFactor{i}$.
Accordingly,
\begin{equation}
u_i^{(t)}
\;\propto\;
\underbrace{\mathbb{E}\!\left[\nabla_{\theta}\mathcal{V}_i(\theta_t)^{\top}\hat{\Delta \theta}_i^{(t)}\right]}_{\displaystyle \ConvertFactor{i}}
\;\cdot\;
\underbrace{\mathbb{E}\!\left[\|\Delta \theta_i^{(t)}\|_2\right]}_{\displaystyle \LearnFactor{i}} .
\label{eq:utility_factorization}
\end{equation}
This factorization motivates a utility that combines update magnitude with evidence that those updates translate into reward gains.
Because neither factor is directly observable online, we approximate each from signals already produced during training: $\LearnFactor{i}$ from advantage magnitudes (Section~\ref{sec:absolute_advantage}), and $\ConvertFactor{i}$ from recent reward gains over a local window (Section~\ref{sec:reward_trend}).

\subsubsection{Advantage-Derived Learnability}
\label{sec:absolute_advantage}

The GRPO policy-gradient loss is a sum of policy-gradient terms weighted by group-normalized advantages.
Appendix~\ref{app:advantage_learnability} gives the supporting bound: when the log-probability gradient norms $\|\nabla_{\theta}\log \pi_{\theta}(o_j\mid x)\|$ vary smoothly across tasks, differences in update magnitude are largely explained by differences in advantage magnitude.
Prompts whose sampled responses receive clearly separated rewards provide stronger learning signals than prompts whose responses receive nearly identical rewards.
At training step $t$, let $\mathcal{H}_t=\{t-W,\ldots,t-1\}$ denote the history window, and let $\mathcal{X}_i^{(t')}$ be the prompts from task $i$ at historical step $t'\in\mathcal{H}_t$.
For each task $i$, we estimate learnability by first computing a step-level mean absolute advantage and then averaging it over this window:
\begin{equation}
\begin{aligned}
a_i^{(t')}
&=
\frac{1}{|\mathcal{X}_i^{(t')}|}
\sum_{x\in \mathcal{X}_i^{(t')}}
\frac{1}{G}\sum_{j=1}^{G} \left| \hat{A}_{t',x,j} \right|, \\
s_i^{(\mathrm{adv})}
&=
\frac{1}{W}
\sum_{t'\in\mathcal{H}_t} a_i^{(t')} .
\end{aligned}
\label{eq:adv_learnability}
\end{equation}
where $\hat{A}_{t',x,j}$ is the group-normalized advantage of rollout $j$ for prompt $x$ at step $t'$.
Following Eq.~\eqref{eq:utility_factorization}, we use $s_i^{(\mathrm{adv})}$ to estimate $\LearnFactor{i}$, which reflects the advantage-driven update that task $i$ can still provide.

\subsubsection{Reward-Derived Progress}
\label{sec:reward_trend}

Advantage magnitude alone does not show whether the policy is improving on a task: a task can induce large GRPO updates even when its task-level expected reward shows no measurable upward trend.
We therefore pair advantage-derived learnability with a signal that tracks recent reward gains.
At the current controller step $t$, this improvement cannot be estimated reliably from a single mixed batch, since task-level reward averages are noisy and some tasks may receive only a few samples.
\pac therefore looks back over the previous $W$ training steps and estimates whether each task's mean reward has been increasing.

Using the same history window $\mathcal{H}_t$, let $\bar{r}_i^{(t')}$ be the step-level mean reward of task $i$ at historical step $t'\in\mathcal{H}_t$, where the mean is taken over the per-rollout rewards $r_j$ defined in Section~\ref{sec:problem_formulation}.
To estimate recent progress, we fit a linear regression model via least squares of the form
\begin{equation}
\bar{r}_i(\tau_{t'}) = b_i + k_i \tau_{t'},
\label{eq:reward_trend}
\end{equation}
where $\tau_{t'}$ is the position of historical step $t'$ normalized to $[0,1]$ within the window.
Let $\hat{k}_i$ be the fitted slope.
We normalize the fitted slopes across tasks:
\begin{equation}
s_i^{(\mathrm{prog})} =
\dfrac{\hat{k}_i}{\sum_{j=1}^{N} |\hat{k}_j| + \epsilon},
\label{eq:prog_norm}
\end{equation}
where $\epsilon$ is a small stability constant.
Under the factorization in Eq.~\eqref{eq:utility_factorization}, $s_i^{(\mathrm{prog})}$ estimates the relative reward trend component: a task receives a larger trend adjustment when its reward improves faster than rewards on the other tasks in the current training phase.

\subsubsection{Progress-Modulated Utility}
\label{sec:fusion}

To combine the two signals, \pac takes advantage-derived learnability as the base utility and rescales it by normalized progress.
Setting $\LearnFactor{i}=s_i^{(\mathrm{adv})}$ and $\ConvertFactor{i}=1+s_i^{(\mathrm{prog})}$ gives
\begin{equation}
u_i^{(t)} = \bigl( 1 + s_i^{(\mathrm{prog})} \bigr) \cdot s_i^{(\mathrm{adv})}.
\label{eq:multiplicative_fusion}
\end{equation}
Here, $s_i^{(\mathrm{adv})}$ estimates the magnitude of the task-induced GRPO update, while $1+s_i^{(\mathrm{prog})}$ modulates this magnitude using recent reward gains.
The unit offset makes the modifier neutral when the normalized progress estimate is near zero, preserving informative advantage signals.

\subsection{Bayesian Curriculum Controller}
\label{sec:thompson}

The fused utility $u_i^{(t)}$ in Eq.~\eqref{eq:multiplicative_fusion} is noisy and is observed with unequal reliability across arms, since different arms receive different numbers of samples in each batch.
\pac therefore maintains a Gaussian belief over the latent utility $v_i$ of each task arm:
\begin{equation}
v_i \sim \mathcal{N}(\mu_i, \sigma_i^2),
\label{eq:gaussian_posterior}
\end{equation}
where the posterior mean $\mu_i$ summarizes the current utility estimate and the posterior variance $\sigma_i^2$ quantifies the remaining uncertainty that drives exploration in Thompson Sampling.
Let $c_i^{(t)}$ be the number of training samples drawn from task $i$ at step $t$, $\bar{c}=B/N$ the count under uniform allocation, and $\rho_i^{(t)}=c_i^{(t)}/\bar{c}$ the relative amount of evidence for task $i$ in that batch.
We treat $\rho_i^{(t)}$ as the observation precision of $u_i^{(t)}$, so that utility estimates supported by more samples are weighted more heavily, while sparsely sampled arms retain higher posterior uncertainty.
Under this Gaussian--Gaussian observation model, the conjugate update is a precision-weighted average:
\begin{equation}
\begin{aligned}
\eta_i
&=
\frac{1}{\sigma_i^{2}}
+
\rho_i^{(t)}, \\
\mu_i^{+}
&=
\frac{
\frac{\mu_i}{\sigma_i^{2}}
+
\rho_i^{(t)} u_i^{(t)}
}{
\eta_i
}, \\
(\sigma_i^{2})^{+}
&=
\eta_i^{-1}.
\end{aligned}
\label{eq:posterior_update}
\end{equation}
The update therefore gives more influence to observations supported by more samples, while preserving higher posterior uncertainty for arms that have been sampled less often.
Task utilities are also non-stationary: an arm that is currently informative may saturate after a few updates, and a previously stalled arm may become learnable as the policy improves.
To prevent the posterior from collapsing onto stale evidence, we apply non-stationary variance inflation after each update:
\begin{equation}
\mu_i \leftarrow \mu_i^{+}, \qquad
\sigma_i^2 \leftarrow (\sigma_i^{2})^{+} + \delta^2,
\label{eq:variance_inflation}
\end{equation}
where the inflation variance $\delta^2$ keeps the posterior variance bounded away from zero, which prevents any single arm from being overprioritized early in training and preserves responsiveness to later shifts in arm utility.

With the posterior in place, the controller builds the next batch of size $B$ by sequential Thompson Sampling.
For each training sample in the batch, the controller draws one utility value $\tilde{v}_j\sim\mathcal{N}(\mu_j,\sigma_j^2)$ for every arm $j$, selects $i=\arg\max_j \tilde{v}_j$, and samples one prompt from $\mathcal{D}_i$; repeating this draw $B$ times produces the full batch.
This per-sample draw lets a single batch mix prompts from multiple arms in proportion to their posteriors, instead of locking the batch onto the arm with the largest current estimate.
This design exploits arms with high posterior mean while keeping exploration proportional to posterior variance; the variance inflation in Eq.~\eqref{eq:variance_inflation} helps preserve opportunities to revisit arms whose utility shifts later in training.

\subsection{Training Procedure}
\label{sec:training_procedure}

Algorithm~\ref{alg:pac} summarizes the training loop, which proceeds in two stages.
At the start of training, reward and advantage statistics for each arm are based on only a few rollouts, so adapting the sampling distribution immediately would amplify the noise in those early estimates.
\pac therefore opens with a cold-start stage of $T_{\mathrm{cold}}$ steps during which batches are drawn uniformly over arms.
The controller still computes the fused utility $u_i^{(t)}$ and updates the Gaussian posterior throughout this stage, so every arm accumulates comparable initial evidence before Thompson Sampling becomes active.

After the cold-start stage, each training step alternates between a policy update and a controller update.
The controller first builds a batch of size $B$ by sequential Thompson Sampling (Section~\ref{sec:thompson}), and the policy then generates GRPO rollouts on the selected prompts and receives verifiable rewards.
The resulting reward and advantage statistics are appended to the history window, after which \pac recomputes the fused utility via Eq.~\eqref{eq:multiplicative_fusion}, applies the precision-weighted posterior update in Eq.~\eqref{eq:posterior_update}, and applies the variance inflation in Eq.~\eqref{eq:variance_inflation} before the next step.

\begin{algorithm}[t]
\caption{\textbf{\pac: Progress-Augmented Advantage Curriculum}}
\label{alg:pac}
\footnotesize
\begin{algorithmic}[1]
\Require Training set with $N$ tasks $\mathcal{D}=\{\mathcal{D}_1,\dots,\mathcal{D}_N\}$; LLM policy $\pi_\theta$ with parameters $\theta$; Batch size $B$; Total training steps $T$; Cold-start length $T_{\mathrm{cold}}$
\State Initialize $(\mu_i,\sigma_i^2)\gets(0,1) \quad \forall i\in\{1,2,\ldots,N\}$
\For{$t\gets 0$ \textbf{to} $T-1$}
\State $\mathcal{B}_t\gets \emptyset$
\While{$|\mathcal{B}_t|<B$}
\If{$t < T_{\mathrm{cold}}$}
\State Select task $i$ uniformly from $\{1,\dots,N\}$
\Else
\State Draw $\tilde{v}_j\sim\mathcal{N}(\mu_j,\sigma_j^2) \quad \forall j\in\{1,\ldots,N\}$
\State Select task $i\gets\arg\max_j\tilde{v}_j$
\EndIf
\State Sample prompt $x$ uniformly from $\mathcal{D}_i$
\State $\mathcal{B}_t\gets \mathcal{B}_t\cup\{x\}$
\EndWhile
\State Run $\pi_\theta$ on each $x\in \mathcal{B}_t$ to generate rollouts $\mathcal{T}$ and compute rewards $\mathbf{r}$
\State Estimate advantages $\hat{A}$ and update $\pi_\theta$ with GRPO
\ForAll{task $i\in \{1,2,\ldots,N\}$}
\State Compute fused utility $u_i^{(t)}$ using Eq.~\eqref{eq:multiplicative_fusion}
\State Update $(\mu_i,\sigma_i^2)$ with $u_i^{(t)}$ using Eq.~\eqref{eq:posterior_update} and Eq.~\eqref{eq:variance_inflation}
\EndFor
\EndFor
\State \Return Fine-tuned LLM $\pi_{\theta}$
\end{algorithmic}
\end{algorithm}

Across the two stages, \pac progressively reallocates rollout budget toward tasks that show both an advantage-derived learning signal and measurable reward improvement.
Thompson Sampling trades off exploitation of high-utility tasks against exploration under posterior uncertainty, while the variance inflation in Eq.~\eqref{eq:variance_inflation} is intended to maintain responsiveness as task utilities evolve during training.

\section{Experiments}
\label{sec:experiments}

\begin{figure*}[!t]
\centering
\begin{subfigure}[t]{0.49\textwidth}
    \centering
    \includegraphics[width=\linewidth]{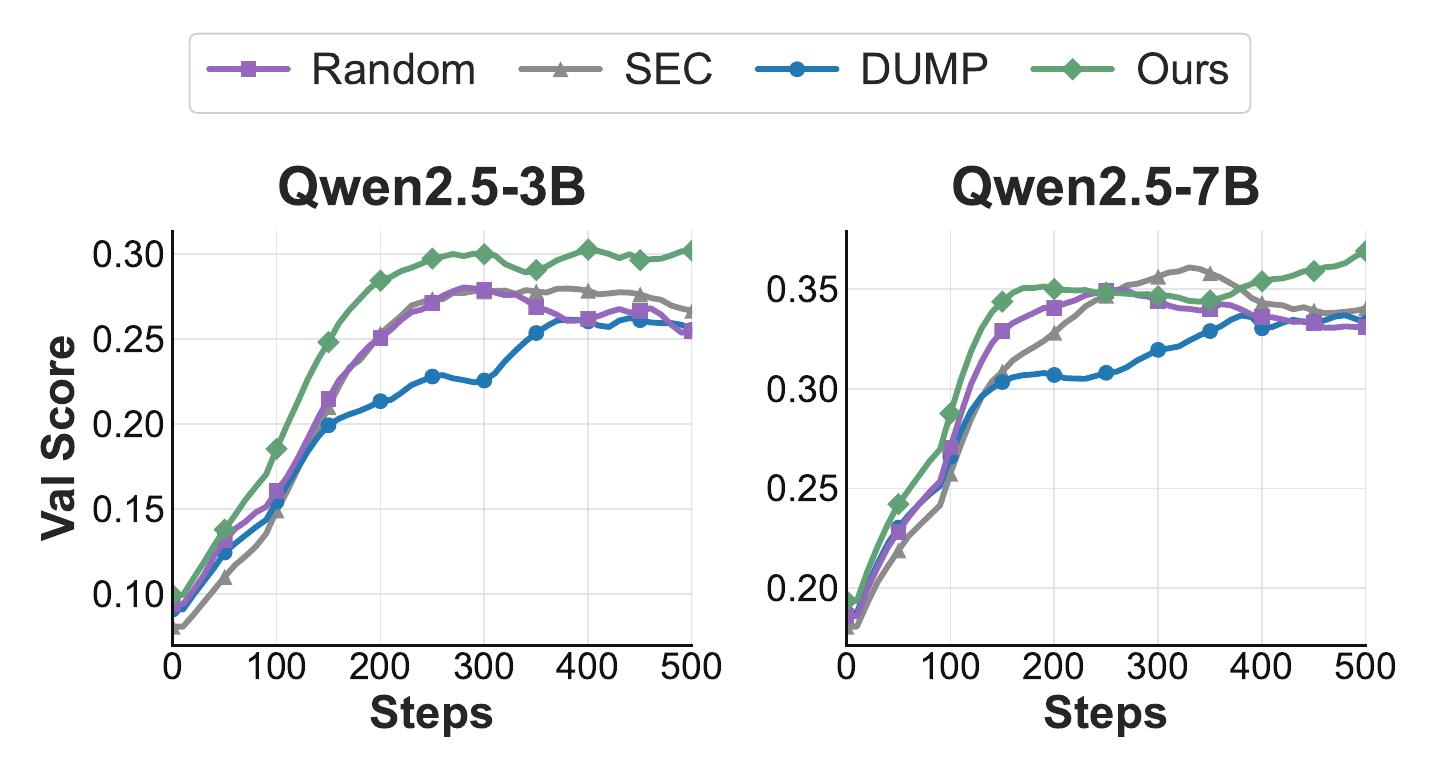}
    \caption{\textbf{Multi-level reasoning setting.}}
    \label{fig:main_results_multi_difficulty}
\end{subfigure}
\hfill
\begin{subfigure}[t]{0.49\textwidth}
    \centering
    \includegraphics[width=\linewidth]{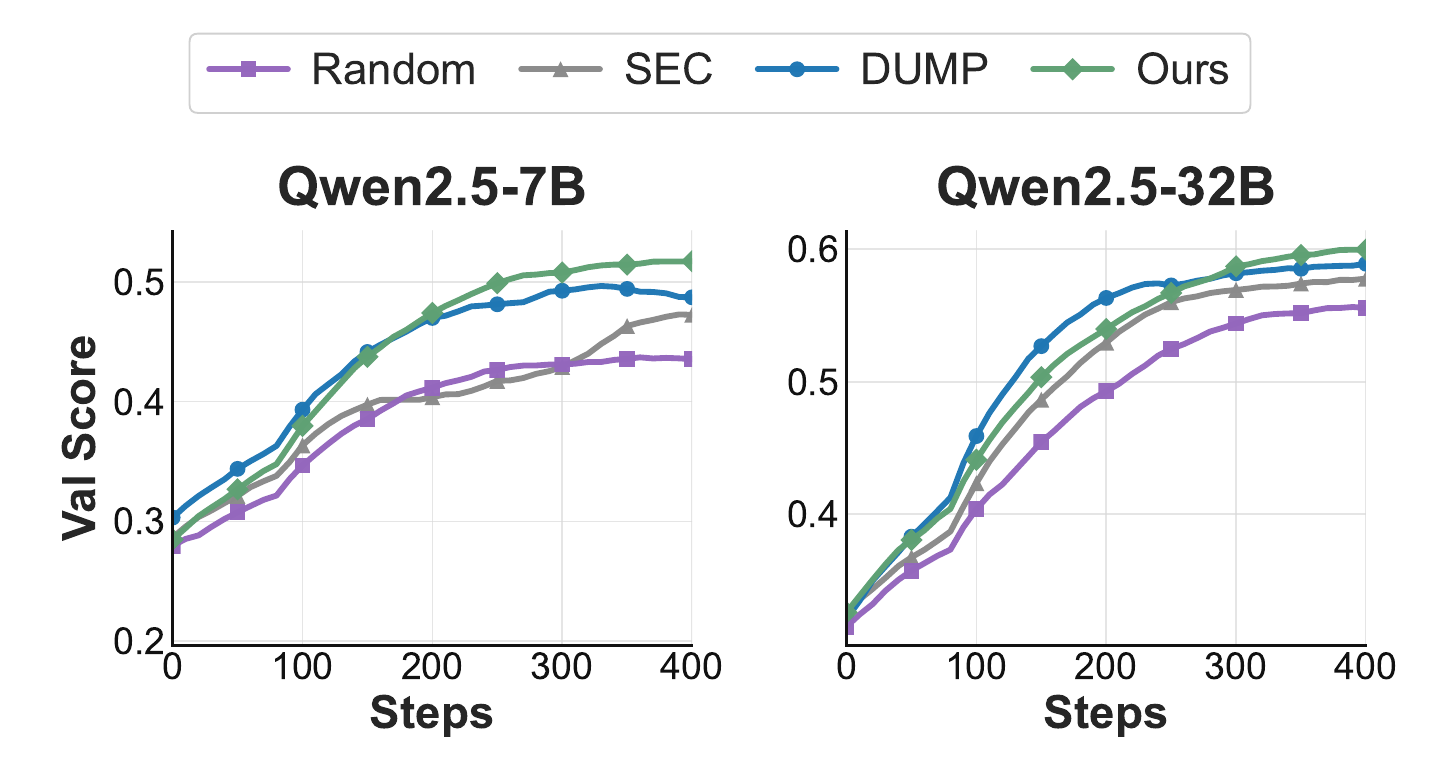}
    \caption{\textbf{Multi-domain reasoning setting.}}
    \label{fig:main_results_multi_domain}
\end{subfigure}
\caption{\textbf{Validation score over training progress for the two multi-task RL settings.} Figure~(a) shows the multi-level reasoning setting with Qwen2.5-3B and Qwen2.5-7B, evaluated on the extremely hard validation splits of Countdown, Zebra, and ARC. Figure~(b) shows the multi-domain reasoning setting with Qwen2.5-7B and Qwen2.5-32B, evaluated on MATH500, AMC22--23, AIME24, BigCodeBench, and the held-out 6-character K\&K split. In both figures, "Val Score" denotes the mean of the task-level validation scores.}
\label{fig:main_results}
\end{figure*}

\subsection{Experimental Setup}
\label{sec:experimental_setup}

\paragraph{Datasets and Evaluation.}
We evaluate \pac in two multi-task RL settings: a multi-level reasoning setting based on \citet{chen2025selfevolving} and a multi-domain mixture spanning mathematics, code generation, and symbolic logic puzzles.

\noindent\textit{Multi-level reasoning setting.}
Following the multi-level task setup of SEC \citep{chen2025selfevolving}, this setting includes Countdown, Zebra, and ARC (Abstraction and Reasoning Corpus) \citep{chollet2019measure}.
Countdown requires the model to combine a given set of integers and arithmetic operators to reach a target number; difficulty is controlled by the number of input integers.
Zebra is a constraint-satisfaction logic puzzle in which entities must be assigned properties from textual clues; difficulty increases with the number of entities and properties.
ARC provides input-output string transformation examples and asks the model to infer the underlying rule for an unseen case; input-string length controls difficulty.
For each task, we use the training portions of the simple, medium, and hard splits for RL post-training.
We evaluate on disjoint held-out validation portions at all four difficulty levels; the main results use the extremely hard level to measure transfer to harder problems, while Appendix~\ref{app:levelwise_validation} reports the complete per-difficulty breakdown.
We run this setting with Qwen2.5-3B and Qwen2.5-7B \citep{qwen2025qwen25}.

\noindent\textit{Multi-domain reasoning setting.}
For the math domain, we train on DAPO-Math-17k, a mathematical reasoning corpus designed for rule-based answer verification in RL training \citep{yu2025dapo}.
For the code domain, we train on Code-R1-12k, where each example provides a programming prompt with unit-test-based ground truth for automatic validation; following \textsc{Omni-Thinker} \citep{li2025omnithinker}, we filter out prompts longer than 1,024 tokens.
For the logic domain, we use the 3-, 4-, and 5-character splits of the K\&K puzzle dataset, where each character is either a knight who always tells the truth or a knave who always lies, and the goal is to infer the identity of each character \citep{wang2025dump}.
For validation, math is evaluated on MATH500 \citep{hendrycks2021math,lightman2023lets}, AMC22--23, and AIME24; code is evaluated by pass@1 on BigCodeBench \citep{zhuo2025bigcodebench}; and K\&K is evaluated on the held-out 6-character split.
Because this setting requires stronger base capabilities across mathematics, code generation, and symbolic logic, we use larger backbones: Qwen2.5-7B and Qwen2.5-32B.
Representative prompt examples for all datasets used in the two settings are provided in Appendix~\ref{app:prompt_examples}.

\paragraph{Baselines.}
We compare against one single-task reference and three curriculum baselines.
\textbf{ST} (Single Task) trains each task or domain separately and evaluates only on its matching validation set.
\textbf{Random} samples training examples uniformly from the available tasks or task splits.
\textbf{SEC} implements the Self-Evolving Curriculum of \citet{chen2025selfevolving}, which formulates curriculum selection as a non-stationary multi-armed bandit over task categories, uses average absolute advantage as the arm reward, and selects categories with Boltzmann Sampling.
\textbf{DUMP} follows the same advantage-driven principle as SEC but performs distribution-level scheduling with UCB-based sampling \citep{wang2025dump}.
Additional implementation details are provided in Appendix~\ref{app:implementation_details}.

\subsection{Results and Discussion}
\label{sec:results_discussion}

\begin{table*}[!t]
\centering
\small
\setlength{\tabcolsep}{4pt}
\renewcommand{\arraystretch}{1.03}
\begin{tabularx}{\linewidth}{@{}l*{4}{>{\centering\arraybackslash}X}@{}}
\toprule
\textbf{Method} & \textbf{Countdown} & \textbf{Zebra} & \textbf{ARC} & \textbf{Avg.} \\
\midrule
\multicolumn{5}{@{}l}{\textbf{Qwen2.5-3B}} \\
Random & $0.285$\stddev{0.021} & $0.251$\stddev{0.016} & $0.229$\stddev{0.013} & $0.255$\stddev{0.010} \\
SEC & $0.271$\stddev{0.019} & $\underline{0.288}$\stddev{0.015} & $\underline{0.240}$\stddev{0.012} & $\underline{0.267}$\stddev{0.009} \\
DUMP & $\underline{0.301}$\stddev{0.017} & $0.263$\stddev{0.014} & $0.202$\stddev{0.011} & $0.255$\stddev{0.009} \\
\rowcolor{ourstableblue}\textbf{Ours} & $\mathbf{0.332}$\stddev{0.022} & $\mathbf{0.303}$\stddev{0.019} & $\mathbf{0.271}$\stddev{0.015} & $\mathbf{0.302}$\stddev{0.011} \\
\stcell{ST} & \stcell{$0.441$\stddev{0.024}} & \stcell{$0.328$\stddev{0.017}} & \stcell{$0.306$\stddev{0.014}} & \stcell{$0.358$\stddev{0.011}} \\
\midrule
\multicolumn{5}{@{}l}{\textbf{Qwen2.5-7B}} \\
Random & $0.393$\stddev{0.023} & $0.294$\stddev{0.017} & $0.306$\stddev{0.016} & $0.331$\stddev{0.011} \\
SEC & $\underline{0.404}$\stddev{0.022} & $\mathbf{0.304}$\stddev{0.015} & $0.313$\stddev{0.014} & $\underline{0.340}$\stddev{0.010} \\
DUMP & $0.375$\stddev{0.021} & $0.294$\stddev{0.018} & $\underline{0.330}$\stddev{0.015} & $0.333$\stddev{0.010} \\
\rowcolor{ourstableblue}\textbf{Ours} & $\mathbf{0.439}$\stddev{0.026} & $\underline{0.296}$\stddev{0.013} & $\mathbf{0.373}$\stddev{0.019} & $\mathbf{0.369}$\stddev{0.012} \\
\stcell{ST} & \stcell{$0.487$\stddev{0.025}} & \stcell{$0.320$\stddev{0.016}} & \stcell{$0.417$\stddev{0.018}} & \stcell{$0.408$\stddev{0.012}} \\
\bottomrule
\end{tabularx}
\caption{\textbf{Final validation performance on the multi-level reasoning setting after 500 training steps.} All results are reported as mean $\pm$ standard deviation over three seeds.}
\label{tab:multi_difficulty_validation}
\end{table*}

\begin{table*}[!t]
\centering
\small
\setlength{\tabcolsep}{2pt}
\renewcommand{\arraystretch}{1.03}
\begin{tabularx}{\linewidth}{@{}l*{6}{>{\centering\arraybackslash}X}@{}}
\toprule
\textbf{Method} & \multicolumn{3}{c}{\textbf{Math}} & \textbf{Code} & \textbf{Logic} & \textbf{Avg.} \\
\cmidrule(lr){2-4}\cmidrule(lr){5-5}\cmidrule(lr){6-6}
& \textbf{MATH500} & \textbf{AMC22--23} & \textbf{AIME24} & \textbf{BigCodeBench} & \textbf{K\&K} & \\
\midrule
\multicolumn{7}{@{}l}{\textbf{Qwen2.5-7B}} \\
Random & $0.727$\stddev{0.020} & $\mathbf{0.489}$\stddev{0.038} & $0.133$\stddev{0.033} & $\underline{0.217}$\stddev{0.013} & $0.640$\stddev{0.016} & $0.435$\stddev{0.014} \\
SEC & $0.535$\stddev{0.021} & $0.387$\stddev{0.036} & $\underline{0.135}$\stddev{0.035} & $0.214$\stddev{0.012} & $0.851$\stddev{0.012} & $0.472$\stddev{0.013} \\
DUMP & $\mathbf{0.745}$\stddev{0.018} & $0.288$\stddev{0.034} & $0.108$\stddev{0.031} & $0.208$\stddev{0.014} & $\underline{0.873}$\stddev{0.015} & $\underline{0.487}$\stddev{0.015} \\
\rowcolor{ourstableblue}\textbf{Ours} & $\underline{0.734}$\stddev{0.022} & $\underline{0.442}$\stddev{0.043} & $\mathbf{0.163}$\stddev{0.039} & $\mathbf{0.218}$\stddev{0.011} & $\mathbf{0.888}$\stddev{0.011} & $\mathbf{0.517}$\stddev{0.013} \\
\stcell{ST} & \stcell{$0.758$\stddev{0.017}} & \stcell{$0.518$\stddev{0.041}} & \stcell{$0.267$\stddev{0.036}} & \stcell{$0.250$\stddev{0.014}} & \stcell{$0.931$\stddev{0.010}} & \stcell{$0.565$\stddev{0.012}} \\
\midrule
\multicolumn{7}{@{}l}{\textbf{Qwen2.5-32B}} \\
Random & $0.648$\stddev{0.018} & $0.598$\stddev{0.040} & $0.246$\stddev{0.034} & $0.247$\stddev{0.013} & $0.923$\stddev{0.011} & $0.556$\stddev{0.013} \\
SEC & $0.820$\stddev{0.016} & $\underline{0.605}$\stddev{0.038} & $0.211$\stddev{0.031} & $0.236$\stddev{0.011} & $0.952$\stddev{0.009} & $0.578$\stddev{0.012} \\
DUMP & $\mathbf{0.833}$\stddev{0.015} & $0.598$\stddev{0.037} & $\underline{0.247}$\stddev{0.032} & $\underline{0.250}$\stddev{0.015} & $\underline{0.957}$\stddev{0.008} & $\underline{0.589}$\stddev{0.014} \\
\rowcolor{ourstableblue}\textbf{Ours} & $\underline{0.822}$\stddev{0.019} & $\mathbf{0.625}$\stddev{0.044} & $\mathbf{0.259}$\stddev{0.028} & $\mathbf{0.253}$\stddev{0.010} & $\mathbf{0.977}$\stddev{0.012} & $\mathbf{0.600}$\stddev{0.011} \\
\stcell{ST} & \stcell{$0.840$\stddev{0.014}} & \stcell{$0.711$\stddev{0.042}} & \stcell{$0.400$\stddev{0.035}} & \stcell{$0.280$\stddev{0.012}} & \stcell{$0.994$\stddev{0.007}} & \stcell{$0.641$\stddev{0.011}} \\
\bottomrule
\end{tabularx}
\caption{\textbf{Final validation performance on the multi-domain reasoning setting after 400 training steps.} All results are reported as mean $\pm$ standard deviation over three seeds. ``Avg.'' first averages the three math benchmarks and then averages Math, Code, and Logic.}
\label{tab:multi_domain_validation}
\end{table*}

Figure~\ref{fig:main_results} reports the validation score over training progress, and Tables~\ref{tab:multi_difficulty_validation}--\ref{tab:multi_domain_validation} report final validation scores.
All numbers in tables and all curves in figures are averaged over three runs with different random seeds.
ST is a single-task reference and is excluded from highlighting.

\textbf{\pac improves both sample efficiency and final performance.}
In Figure~\ref{fig:main_results}, \pac reaches the final average scores of the strongest curriculum baselines with fewer rollout steps.
In the multi-level setting, \pac matches the final average of SEC after about 170 steps with Qwen2.5-3B and 143 steps with Qwen2.5-7B, whereas SEC reaches its own final average at about 225 and 227 steps, corresponding to sample-efficiency gains of roughly $1.3\times$ and $1.6\times$.
In the multi-domain setting, \pac reaches the final average of DUMP earlier for both Qwen2.5-7B and Qwen2.5-32B, corresponding to approximately $1.2\times$ and $1.3\times$ sample-efficiency gains, respectively, while maintaining higher final averages.
This behavior is consistent with the utility design: advantage magnitude identifies tasks that can still provide informative updates, while reward-derived progress discounts updates that no longer lead to reward gains.

\textbf{Multi-level results show more balanced transfer to harder splits.}
Table~\ref{tab:multi_difficulty_validation} shows that \pac achieves the best average score for both model sizes, with improvements over SEC of 13.1\% on Qwen2.5-3B and 8.5\% on Qwen2.5-7B.
These gains indicate that \pac improves cross-difficulty transfer rather than only optimizing one validation split.
SEC and DUMP remain competitive on some individual splits, suggesting that advantage-based curricula can identify locally useful update signals.
However, their strengths are less consistent across model sizes and task types: high-advantage splits may continue to receive budget after reward-derived progress slows, whereas \pac shifts budget toward splits with active gains.

\textbf{Multi-domain results show stronger cross-domain balance.}
Table~\ref{tab:multi_domain_validation} shows that \pac achieves the best curriculum average for both Qwen2.5-7B and Qwen2.5-32B, with improvements over DUMP of 6.2\% and 1.9\%.
DUMP and Random remain strong on selected math benchmarks, but these localized gains do not translate into the best cross-domain average.
By contrast, \pac maintains stronger average performance while improving on harder math benchmarks as well as on code and logic benchmarks, suggesting that the controller avoids overcommitting rollout budget to a single domain.
Appendix~\ref{app:curriculum_dynamics} further analyzes the training-time curriculum dynamics in the multi-domain setting and shows how \pac reallocates rollout budget as reward-derived progress changes.
Overall, these results suggest that effective curriculum allocation should jointly track where advantages remain large and where recent updates continue to yield measurable reward gains.

\subsection{Ablation Study}
\label{sec:ablation_study}

\begin{figure}[t]
\centering
\includegraphics[width=0.69\linewidth]{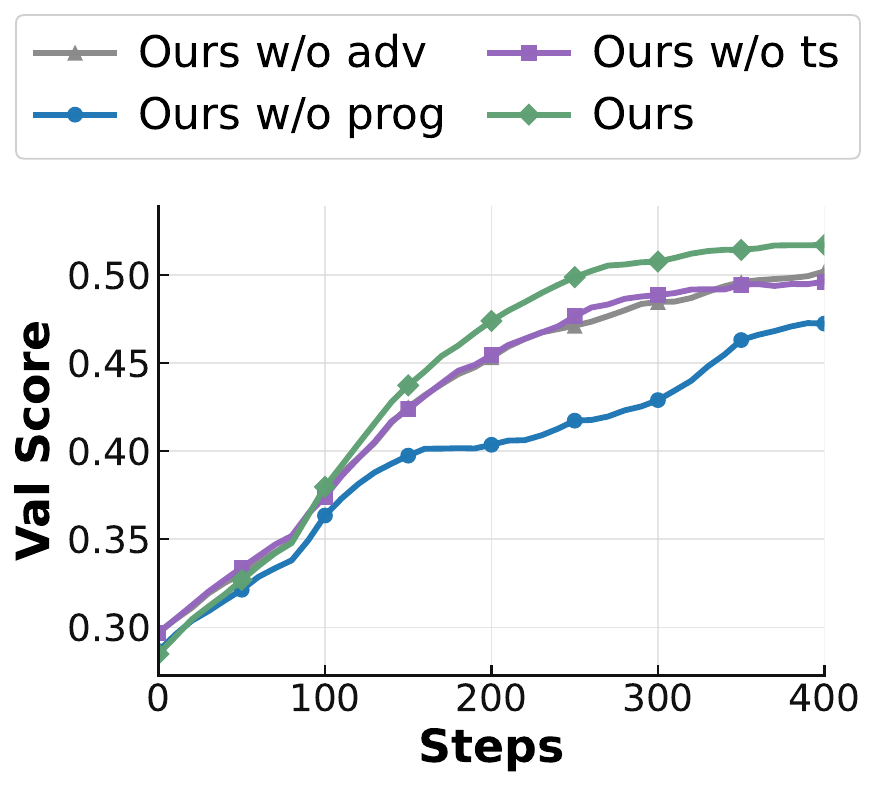}
\caption{\textbf{Ablation of \pac on the multi-domain reasoning setting with Qwen2.5-7B.} "Val Score" denotes the mean of the task-level validation scores.}
\label{fig:ablation}
\end{figure}

We ablate \pac's three design choices on the multi-domain reasoning setting with Qwen2.5-7B: \textbf{Ours w/o adv} sets $u_i^{(t)}=s_i^{(\mathrm{prog})}$, \textbf{Ours w/o prog} sets $u_i^{(t)}=s_i^{(\mathrm{adv})}$, and \textbf{Ours w/o ts} keeps the fused utility but replaces the Bayesian controller with Boltzmann Sampling over $u_i^{(t)}$.
Appendix~\ref{app:ablation_details} formalizes each variant and reports the per-task breakdown.
History-window sensitivity is reported in Appendix~\ref{app:history_window_sensitivity}.

All three components contribute, with reward-derived progress the dominant factor: \pac improves over Ours w/o prog by 9.5\% on the multi-domain average.
This gap arises because the controller collapses to advantage-only allocation and keeps investing in arms whose advantages remain large after their rewards have saturated, the K\&K Logic over-allocation visualized in Appendix~\ref{app:curriculum_dynamics}.
The improvement over Ours w/o adv is 3.0\%: a short-window reward slope alone is unreliable on slow-progress benchmarks such as AIME24 and BigCodeBench, where the learnability term separates improvement from noisy slopes.
The improvement over Ours w/o ts is 4.2\%: the fused utility is preserved but posterior uncertainty is discarded, so once a few arms develop higher utility estimates the allocation concentrates on them quickly.

\section{Conclusion}

In this work, we study online curriculum learning for multi-task RL post-training of LLMs, where rollout budgets must be allocated across heterogeneous tasks whose utility shifts during training.
Advantage-only curricula capture the induced GRPO update magnitude but ignore whether it translates into reward gains.
We therefore propose \pac, a Progress-Augmented Advantage Curriculum that estimates task utility by combining advantage-derived learnability with reward-derived progress, and uses Bayesian Thompson Sampling to adapt task allocation under uncertainty.
Across multi-level and multi-domain reasoning settings, \pac reaches baseline-level validation averages with fewer rollout steps than random and advantage-based baselines.
\pac improves final performance, achieving higher validation averages across model sizes and settings.
These findings suggest that adaptive RL post-training curricula benefit from tracking both update potential and reward conversion.

\section*{Limitations}

\paragraph{Reward signal assumptions.}
\pac is validated under RL post-training with outcome-level verifiable rewards, including rule-based answer checkers, symbolic verifiers, and unit-test execution.
This setting matches many reasoning-oriented GRPO and RLVR pipelines and provides relatively dense reward observations from which window-level mean rewards can be reliably computed.
Under sparser, noisier, or longer-horizon reward signals, such as those from learned reward models or process-level rewards, the linear trend estimator in Eq.~\eqref{eq:reward_trend} may need to be replaced by a more robust progress estimator that is less sensitive to reward noise and non-monotonic learning dynamics.

\paragraph{Task-level granularity.}
\pac maintains a Gaussian posterior per task arm, so the curriculum granularity is tied to the available task split structure, such as domain labels or difficulty buckets.
When a single arm contains substantial internal diversity, for example, a math arm spanning a wide range of prompt difficulties, task-level allocation cannot exploit the variation within that arm.
Extending \pac from task-level to prompt-level or instance-level sampling remains open and would require posterior structures that scale to arm counts far larger than the three and nine arms studied here.

\paragraph{Empirical scope.}
Our experiments evaluate \pac with backbones up to 32B parameters on reasoning, code, and logic mixtures with at most nine task arms.
We did not directly study \pac on other capability families such as instruction following, dialogue, multilingual, or safety, on substantially larger arm sets, or under GRPO variants with alternative advantage normalization methods.
We expect the multiplicative fusion in Eq.~\eqref{eq:multiplicative_fusion} and the Thompson Sampling controller in Section~\ref{sec:thompson} to extend to these settings, but verifying this empirically remains future work.

\bibliography{paper}
\clearpage
\appendix
\numberwithin{equation}{section}
\renewcommand{\theequation}{\Alph{section}.\arabic{equation}}
\renewcommand{\theHequation}{\Alph{section}.\arabic{equation}}

\onecolumn
\section{Advantage-Derived Learnability Derivation}
\label{app:advantage_learnability}

\subsection{Statement and Assumptions}
\begin{theorem}[Absolute advantage as a learnability proxy]
Fix the current policy $\pi_{\theta}$ and a task dataset $\mathcal{D}_i$.
For prompts $x\sim\mathcal{D}_i$ and rollouts $\{o_j\}_{j=1}^{G}\sim\pi_{\theta}(\cdot\mid x)$, define
\[
L_i(\theta)=\mathbb{E}_{x\sim\mathcal{D}_i,\, o_{1:G}\sim\pi_{\theta}(\cdot\mid x)}\left[\frac{1}{G}\sum_{j=1}^{G}|\hat{A}_j|\right].
\]
When the score-gradient norms $\|\nabla_{\theta}\log\pi_{\theta}(o_j\mid x)\|$ are bounded and vary smoothly across tasks, $L_i(\theta)$ tracks the scale of the GRPO update induced by task $i$ up to slowly varying factors.
Thus the mean absolute advantage provides an observable proxy for the learnability factor $\LearnFactor{i}$ used in Section~\ref{sec:absolute_advantage}.
\end{theorem}

\subsection{Proof and Online Estimator}
\begin{proof}
For a prompt $x$ and its group of sampled responses $\{o_j\}_{j=1}^{G}$, the group-normalized advantage is
\[
\hat{A}_j=\frac{r_j-\operatorname{mean}(\{r_1,\ldots,r_G\})}{\operatorname{std}(\{r_1,\ldots,r_G\})+\epsilon},
\]
matching the definition in Section~\ref{sec:problem_formulation}.
The policy gradient under common policy-gradient methods (e.g., PPO or GRPO) can be written as
\begin{equation}
g_x(\theta)=\mathbb{E}_{o_{1:G}\sim \pi_{\theta}(\cdot \mid x)}\left[\frac{1}{G}\sum_{j=1}^{G}\hat{A}_j\,\nabla_{\theta}\log \pi_{\theta}(o_j \mid x)\right].
\label{eq:pg_objective_adv}
\end{equation}
Aggregating over prompts from task $i$ gives the task-conditioned gradient $g_i(\theta)=\mathbb{E}_{x\sim\mathcal{D}_i}[g_x(\theta)]$.
Its scale is controlled by the absolute advantage and the score-gradient norm:
\begin{equation}
\left\|g_i(\theta)\right\|_2
\;\lesssim\;
\mathbb{E}_{x\sim\mathcal{D}_i,\, o_{1:G}\sim\pi_{\theta}(\cdot\mid x)}
\!\left[
\frac{1}{G}\sum_{j=1}^{G}
\left|\hat{A}_j\right|\,
\left\|\nabla_{\theta}\log\pi_{\theta}(o_j\mid x)\right\|_2
\right].
\label{eq:abs_adv_grad_mag}
\end{equation}
The right-hand side separates the task-conditioned update scale into two factors: the absolute advantage $|\hat{A}_j|$ and the score-gradient norm $\|\nabla_{\theta}\log\pi_{\theta}(o_j\mid x)\|_2$.
Under the smoothness assumption in the theorem statement, the score-gradient norm changes more slowly across tasks than the reward separation captured by the normalized advantages.
The main task-dependent term controlling the GRPO update scale is therefore the mean absolute advantage $L_i(\theta)$.
Over the history window, this quantity is estimated by
\[
\begin{aligned}
a_i^{(t')}
&=
\frac{1}{|\mathcal{X}_i^{(t')}|}
\sum_{x\in\mathcal{X}_i^{(t')}}
\frac{1}{G}\sum_{j=1}^{G}|\hat{A}_{t',x,j}|, \\
s_i^{(\mathrm{adv})}
&=
\frac{1}{W}
\sum_{t'=t-W}^{t-1}a_i^{(t')} .
\end{aligned}
\]
which is the advantage-derived learnability score in Eq.~\eqref{eq:adv_learnability}.
Thus $s_i^{(\mathrm{adv})}$ estimates $\LearnFactor{i}$: the amount of advantage-driven update signal that task $i$ can still provide under the current policy.
\end{proof}

\clearpage

\twocolumn
\section{Implementation Details}
\label{app:implementation_details}

\subsection{Data, Arms, and Verifiers}
Table~\ref{tab:dataset_arm_reward_stats} summarizes the arm definitions, data sizes, validation splits, and automatic reward functions for the two experimental settings.

\begin{table*}[t]
\centering
\small
\setlength{\tabcolsep}{3.0pt}
\renewcommand{\arraystretch}{1.12}
\begin{tabularx}{\linewidth}{@{}>{\raggedright\arraybackslash}p{0.15\linewidth}>{\raggedright\arraybackslash}p{0.23\linewidth}>{\raggedright\arraybackslash}p{0.29\linewidth}>{\raggedright\arraybackslash}X@{}}
\toprule
\textbf{Setting} & \textbf{Training arms} & \textbf{Data size and validation} & \textbf{Reward / verifier} \\
\midrule
Multi-level &
9 arms from \{Countdown, Zebra, ARC\} $\times$ difficulty levels 1--3, defined by \texttt{extra\_info.type} and \texttt{extra\_info.}\allowbreak\texttt{difficulty}. &
90,000 training examples in total, with 10,000 examples per arm. Validation files contain 800 examples per task, with 200 examples for each difficulty level 1--4; reported out-of-distribution validation uses level 4. &
Countdown checks whether the boxed expression uses each number exactly once and reaches the target. Zebra and ARC compare the boxed answer against the rule-based target. \\
\midrule
Multi-domain &
3 arms from \{Math, Code, K\&K Logic\}, defined by \texttt{extra\_info.type}. K\&K 3-, 4-, and 5-character splits are grouped into the logic arm. &
26,657 training examples: 11,499 from DAPO-Math, 12,458 from Code-R1, and 2,700 from K\&K Logic. Validation uses 500 MATH500 examples, 83 AMC examples, 30 AIME examples, BigCodeBench for code, and 100 K\&K 6-character examples. &
Math uses the DeepScaleR-style symbolic and numeric answer verifier. Code uses unit-test execution for Code-R1 and BigCodeBench tasks. K\&K uses an answer verifier for knight/knave assignments. \\
\bottomrule
\end{tabularx}
\caption{\textbf{Dataset, arm, and reward statistics for the two experimental settings.}}
\label{tab:dataset_arm_reward_stats}
\end{table*}

\subsection{Training and Sampling}

\paragraph{Training configuration.}
All experiments are run on a single node with 8 NVIDIA GPUs.
At each training step, the curriculum sampler first constructs a batch with size $B=256$ in both the multi-level and multi-domain settings.
For each selected prompt, the current policy samples $G=8$ responses, which are scored by the corresponding verifiable reward function.
Group-normalized advantages are computed only within the $G$ responses generated for the same prompt, so reward scales are not normalized across different tasks or domains before the task-level statistics are accumulated.
The actor optimizer uses a PPO mini-batch size of 64.
Prompts are truncated to 2,048 tokens in both the multi-level and multi-domain settings; generated responses are truncated to 4,096 tokens in both settings.
Within each experimental setting, these GRPO hyperparameters are kept fixed across model sizes, curriculum baselines, and ablation variants.

\paragraph{Task arms and sampling.}
The curriculum controller operates over the training splits described in Section~\ref{sec:experimental_setup}; all multi-task methods receive the same arm set in a given experimental setting.
When an arm is selected, we draw one prompt uniformly from the corresponding training split.
For the multi-level setting, the arms follow the available training difficulty splits.
For the multi-domain setting, the arms follow the domain-specific training splits used to form the math, code, and logic mixture.
For all multi-task runs, every prompt is drawn through the corresponding curriculum sampler.
Table~\ref{tab:dataset_arm_reward_stats} summarizes the task arms, data sizes, and reward functions used in the two experimental settings.
In the multi-domain setting, the source corpora naturally provide different numbers of filtered, verifier-compatible examples.
We retain each filtered corpus as the prompt pool for its domain; per-step domain exposure is determined by the arm-level sampler, so corpus size affects only which prompts can be drawn after an arm is selected.

\subsection{Controller State and Utility Estimation}

\paragraph{\pac state updates.}
We initialize each arm in \pac with a Gaussian posterior mean $\mu_i=0$ and variance $\sigma_i^2=1$.
The cold-start length is fixed to $T_{\mathrm{cold}}=50$ steps in both the multi-level and multi-domain settings.
During this stage, batch construction is uniform, but the controller still records rewards and advantages and updates the posterior, giving each arm a comparable amount of initial evidence.
After the cold-start stage, batches are formed by the sequential Thompson Sampling procedure in Section~\ref{sec:thompson}: for each training sample, the controller draws one utility value from every arm posterior and assigns the sample to the arm with the largest draw.
For each arm, the resulting count $c_i^{(t)}$ determines the evidence weight $\rho_i^{(t)}=c_i^{(t)}/(B/N)$ in Eq.~\eqref{eq:posterior_update}.
After each posterior update, we apply the non-stationary variance inflation in Eq.~\eqref{eq:variance_inflation} to preserve exploration as task utilities change during training.
We use inflation variance $\delta^2=0.02$ and set all numerical stability constants $\epsilon$ in the advantage normalization and Eq.~\eqref{eq:prog_norm} to $10^{-9}$ for all experiments.

\paragraph{Utility estimation.}
We compute the advantage-derived learnability score $s_i^{(\mathrm{adv})}$ as the average step-level mean absolute group-normalized advantage over prompts from arm $i$ in the history window, and the reward-derived progress estimate in Eq.~\eqref{eq:reward_trend} uses the same window over the step-level mean reward of each arm.
When an arm has too few recent observations for a reliable trend fit, we set its fitted slope to zero before cross-task normalization, so that the \pac controller relies only on training-time rewards, advantages, and sampling counts, with validation scores reserved for reporting.
The window length $W$ sets the time scale of these estimates: a short window is overly sensitive to individual rollout batches, whereas a long window lags behind saturation or delayed improvement.
We use $W=16$ in all experiments, which balances smoothing single-batch reward noise against tracking non-stationary changes in task utility; the trend fit, cross-task slope normalization, and zero-slope fallback further reduce sensitivity to the exact window length.
Combined with the cold-start length above, this single window choice keeps the controller on a consistent time scale across settings, rather than tuning curriculum hyperparameters to validation benchmarks.

\subsection{Validation Performance at Individual Difficulty Levels}
\label{app:levelwise_validation}

The main multi-level evaluation in Table~\ref{tab:multi_difficulty_validation} focuses on the extremely hard split to measure transfer beyond the training difficulty range.
To test whether the improvement is confined to that split, Table~\ref{tab:levelwise_validation} reports final scores on held-out validation portions at all four difficulty levels, averaged across Countdown, Zebra, and ARC.
All validation examples are disjoint from the RL training data, including the simple, medium, and hard levels used during training.

\begin{table}[H]
\centering
\small
\setlength{\tabcolsep}{3.5pt}
\renewcommand{\arraystretch}{1.03}
\begin{tabularx}{\columnwidth}{@{}l*{4}{>{\centering\arraybackslash}X}@{}}
\toprule
\textbf{Method} & \textbf{Simple} & \textbf{Medium} & \textbf{Hard} & \textbf{Ext.} \\
\midrule
Random & 0.637 & 0.479 & 0.409 & 0.255 \\
SEC & 0.643 & 0.488 & \underline{0.413} & \underline{0.267} \\
DUMP & \underline{0.732} & \textbf{0.558} & 0.412 & 0.255 \\
\rowcolor{ourstableblue}\textbf{Ours} & \textbf{0.734} & \underline{0.552} & \textbf{0.466} & \textbf{0.302} \\
\bottomrule
\end{tabularx}
\caption{\textbf{Validation performance at individual difficulty levels in the multi-level setting.} Scores are measured after 500 training steps with Qwen2.5-3B and averaged across Countdown, Zebra, and ARC. ``Ext.'' denotes the extremely hard level.}
\label{tab:levelwise_validation}
\end{table}

Table~\ref{tab:levelwise_validation} shows that \pac preserves performance at lower difficulty while its largest gains emerge at higher difficulty.
Relative to the strongest baseline in each column, \pac is within 1.1\% on Simple and Medium, while improving Hard and Ext. by 12.8\% and 13.1\%, respectively.
This pattern indicates that the adaptive curriculum does not trade away performance on easier levels for out-of-distribution transfer; instead, it yields its clearest benefit where the reasoning problems are more difficult.

\subsection{History-Window Sensitivity}
\label{app:history_window_sensitivity}

We examine the sensitivity of \pac to the history-window length $W$, which determines the time scale of advantage smoothing in Eq.~\eqref{eq:adv_learnability} and reward-trend estimation in Eq.~\eqref{eq:reward_trend}.
We evaluate six window lengths, $W\in\{4,16,32,64,128,256\}$, in the nine-arm multi-level setting with Qwen2.5-3B, holding all other settings fixed.

\begin{table}[H]
\centering
\small
\setlength{\tabcolsep}{1.5pt}
\renewcommand{\arraystretch}{1.03}
\begin{tabularx}{\columnwidth}{@{}c*{4}{>{\centering\arraybackslash}X}@{}}
\toprule
\textbf{$W$} & \textbf{Countdown} & \textbf{Zebra} & \textbf{ARC} & \textbf{Avg.} \\
\midrule
4 & 0.245 & 0.240 & 0.228 & 0.238 \\
\rowcolor{ourstableblue}16 & 0.332 & 0.303 & 0.271 & 0.302 \\
32 & 0.326 & 0.298 & 0.264 & 0.296 \\
64 & 0.288 & 0.271 & 0.253 & 0.271 \\
128 & 0.252 & 0.248 & 0.232 & 0.244 \\
256 & 0.248 & 0.244 & 0.227 & 0.240 \\
\bottomrule
\end{tabularx}
\caption{\textbf{Sensitivity to the history-window length $W$.} We report validation accuracy after 500 training steps on the held-out extremely hard splits in the multi-level setting with Qwen2.5-3B. $W=16$ is the default.}
\label{tab:history_window_sensitivity}
\end{table}

Table~\ref{tab:history_window_sensitivity} reveals a trade-off between estimation noise and responsiveness.
\textbf{An overly short window is noisy.}
With $W=4$, the average drops by 21.2\% relative to the default (0.238 vs.\ 0.302).
This window averages the advantage signal over only four recent steps and fits the reward slope from only four task-level reward observations.
Because individual mixed batches can have noisy task rewards and few samples for some arms (Section~\ref{sec:reward_construction}), the controller can respond to transient fluctuations rather than sustained learning progress.

\textbf{An overly long window adapts too slowly.}
Increasing $W$ to 64, 128, and 256 reduces the average by 10.3\%, 19.2\%, and 20.5\%, respectively, relative to $W=16$.
Since task utility is non-stationary, a long window retains observations generated by substantially earlier policies.
Both the smoothed advantage estimate and the fitted reward trend can therefore lag behind task saturation or newly emerging progress, delaying the reallocation that \pac is designed to perform.

\textbf{\pac is stable near the default.}
$W=32$ reaches 0.296, only 2.0\% below the default 0.302, with similarly small differences on all three tasks.
Thus, $W=16$ provides the best balance between smoothing rollout noise and tracking changing task utility, while the result at $W=32$ shows that \pac is not narrowly tuned to a single window value.

\section{Ablation Study Details}
\label{app:ablation_details}

This appendix formalizes the three ablation variants used in Section~\ref{sec:ablation_study} and reports their per-task validation performance.
All variants share the same multi-domain training data, cold-start length $T_{\mathrm{cold}}$, history window $W$, and GRPO hyperparameters as the full \pac configuration.
They differ from the full model in exactly one place: either the utility observation that feeds the controller, or the rule used to convert utilities into rollout allocation.

\subsection{Variant Definitions}

\paragraph{Utility ablations: \textit{Ours w/o adv} and \textit{Ours w/o prog}.}
The two utility ablations leave the Gaussian posterior in Eq.~\eqref{eq:gaussian_posterior} and the Thompson Sampling controller of Section~\ref{sec:thompson} unchanged, and modify only the utility observation that drives the posterior update.
Recall from Eq.~\eqref{eq:multiplicative_fusion} that the full model uses the multiplicative fusion
\[
u_i^{(t)} \;=\; \bigl(1 + s_i^{(\mathrm{prog})}\bigr) \cdot s_i^{(\mathrm{adv})},
\]
in which $s_i^{(\mathrm{adv})}$ estimates the learnability factor $\LearnFactor{i}$ and $1+s_i^{(\mathrm{prog})}$ estimates the conversion factor $\ConvertFactor{i}$ in the factorization of Eq.~\eqref{eq:utility_factorization}.
The two utility ablations isolate each factor:
\begin{align}
u_{i,\text{w/o adv}}^{(t)} &\;=\; s_i^{(\mathrm{prog})}, \label{eq:ablation_no_adv}\\
u_{i,\text{w/o prog}}^{(t)} &\;=\; s_i^{(\mathrm{adv})}. \label{eq:ablation_no_prog}
\end{align}
Eq.~\eqref{eq:ablation_no_adv} drops the learnability factor and lets reward-derived progress drive allocation on its own.
Eq.~\eqref{eq:ablation_no_prog} drops the conversion factor and recovers an advantage-only utility, which is the same type of signal used by SEC and DUMP.

\paragraph{Controller ablation: \textit{Ours w/o ts}.}
The controller ablation leaves the fused utility of Eq.~\eqref{eq:multiplicative_fusion} unchanged and replaces only the allocation rule.
Concretely, \textbf{Ours w/o ts} discards the Gaussian posterior of Eq.~\eqref{eq:gaussian_posterior} and assigns each prompt in the rollout batch to arm $i$ according to a Boltzmann distribution over the current utility scores,
\begin{equation}
p_i^{(t)} \;=\; \frac{\exp\!\left(u_i^{(t)}\right)}{\sum_{j=1}^{N}\exp\!\left(u_j^{(t)}\right)}.
\label{eq:boltzmann_ablation}
\end{equation}
Under this rule, the controller does not model its own uncertainty: the precision-weighted posterior update in Eq.~\eqref{eq:posterior_update} is unused, and the variance inflation in Eq.~\eqref{eq:variance_inflation} is inactive.
Arms with larger utilities receive higher sampling probabilities, but the allocation is deterministic given the utility scores and cannot trade exploitation against posterior variance.

\subsection{Per-Task Results}

\begin{table*}[t]
\centering
\small
\setlength{\tabcolsep}{2pt}
\renewcommand{\arraystretch}{1.03}
\begin{tabularx}{\linewidth}{@{}l*{6}{>{\centering\arraybackslash}X}@{}}
\toprule
\textbf{Method} & \multicolumn{3}{c}{\textbf{Math}} & \textbf{Code} & \textbf{Logic} & \textbf{Avg.} \\
\cmidrule(lr){2-4}\cmidrule(lr){5-5}\cmidrule(lr){6-6}
& \textbf{MATH500} & \textbf{AMC22--23} & \textbf{AIME24} & \textbf{BigCodeBench} & \textbf{K\&K} & \\
\midrule
\textbf{Ours w/o adv} & 0.721 & \textbf{0.462} & 0.128 & 0.213 & \underline{0.856} & \underline{0.502} \\
\textbf{Ours w/o prog} & 0.535 & 0.387 & 0.135 & \underline{0.214} & 0.851 & 0.472 \\
\textbf{Ours w/o ts} & \underline{0.726} & \underline{0.450} & \underline{0.141} & 0.213 & 0.837 & 0.496 \\
\rowcolor{oursgray}\textbf{Ours} & \textbf{0.734} & 0.442 & \textbf{0.163} & \textbf{0.218} & \textbf{0.888} & \textbf{0.517} \\
\bottomrule
\end{tabularx}
\caption{\textbf{Per-task ablation results on the multi-domain reasoning setting with Qwen2.5-7B after 400 training steps.} "Avg." first averages the three math benchmarks and then takes the mean over Math, Code, and Logic.}
\label{tab:ablation_per_task}
\end{table*}

Table~\ref{tab:ablation_per_task} reports the final validation accuracy of each variant on the five multi-domain benchmarks.
Three patterns are consistent with the analysis in Section~\ref{sec:ablation_study}.
First, \textbf{Ours w/o prog} loses most ground on MATH500 and AMC22--23, where it undersamples the Math arm because the saturating K\&K arm continues to dominate an advantage-only utility; this behavior is analyzed in Appendix~\ref{app:curriculum_dynamics}.
Second, \textbf{Ours w/o adv} retains the progress signal and stays close to the full model on the easier benchmarks, but loses accuracy on AIME24 and BigCodeBench, the two splits where reward gains are slow and the progress estimator is noisier; the learnability term helps distinguish slow reward gains from noisy slopes in this regime.
Third, \textbf{Ours w/o ts} keeps the fused utility but records the lowest K\&K accuracy among the variants, because Boltzmann Sampling collapses the K\&K share once Math and Code utilities exceed it, and the resulting rule has no analogue of variance inflation to revisit the arm.

\section{Analysis of Curriculum Dynamics in Multi-Domain Training}
\label{app:curriculum_dynamics}

\begin{figure*}[t]
\centering
\includegraphics[width=0.95\linewidth]{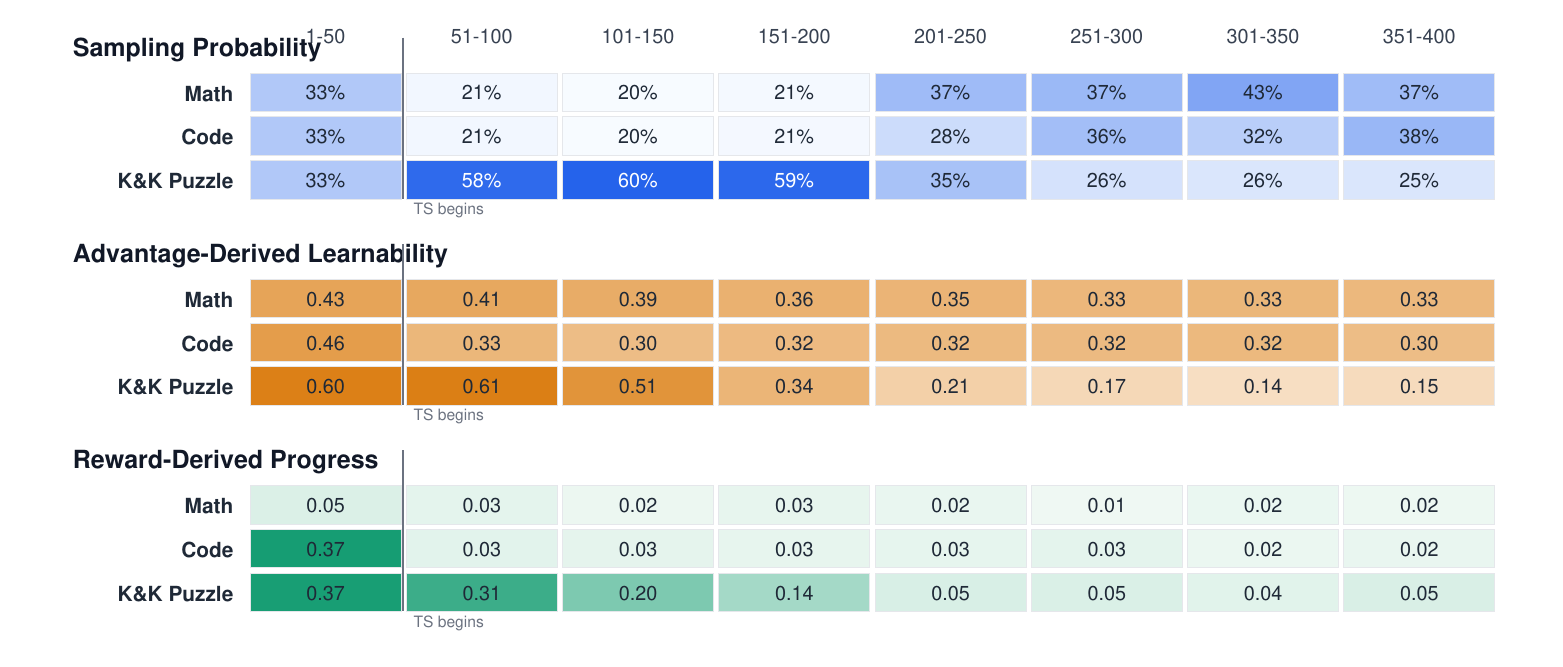}
\caption{\textbf{Training-time curriculum dynamics of \pac on the multi-domain setting with Qwen2.5-7B.} Columns report, for each task arm, the sampling probability, the advantage-derived learnability $s_i^{(\mathrm{adv})}$, and the reward-derived progress $s_i^{(\mathrm{prog})}$; rows correspond to the Math, Code, and K\&K Logic arms. Each cell aggregates the mean over a 50-step non-overlapping plotting window. The vertical marker indicates the end of the uniform cold-start stage at step 50, after which Thompson Sampling becomes the allocation rule.}
\label{fig:pac_group2_7b_heatmaps}
\end{figure*}

Figure~\ref{fig:pac_group2_7b_heatmaps} exposes the controller-internal trajectory of \pac during multi-domain training of Qwen2.5-7B: per-arm sampling probability, advantage-derived learnability $s_i^{(\mathrm{adv})}$, and reward-derived progress $s_i^{(\mathrm{prog})}$, each averaged within a 50-step plotting window. Reading the figure left to right shows how the rollout budget moves between arms; reading it across the two utility columns shows which signal is causing that movement. We use the figure for two purposes: to verify that \pac discovers a non-trivial curriculum once Thompson Sampling activates at step 50, and \textbf{to attribute the multi-domain validation gain in Table~\ref{tab:multi_domain_validation} to a concrete reallocation mechanism} rather than to unattributed optimizer stochasticity.

\subsection{Three-Phase Allocation Trajectory}

The sampling-probability column separates into three phases. The cold-start window allocates roughly a third of the rollout budget to each arm, so every arm accumulates comparable initial evidence before the posterior update in Eq.~\eqref{eq:posterior_update} starts to differentiate them. Once Thompson Sampling becomes active, the controller concentrates budget on K\&K Logic, whose share rises to about 60\% within roughly the next 100 training steps; during this phase K\&K Logic records both the highest advantage-derived learnability and the highest reward-derived progress, so the fused utility in Eq.~\eqref{eq:multiplicative_fusion} ranks it first under both signals simultaneously. After step 200, \textbf{the budget rotates back toward Math and Code, and the K\&K Logic share settles near one-quarter of the batch} for the remainder of training.

\subsection{Reward-Derived Progress Drives the Reallocation}

The K\&K Logic arm reveals why a controller needs more than advantage magnitude. Its advantage-derived learnability $s_i^{(\mathrm{adv})}$ decreases by roughly half over training but remains comparable to the other two arms, so an advantage-only scheduler would continue to over-allocate rollouts to K\&K Logic. The reward-derived progress signal $s_i^{(\mathrm{prog})}$ decays much more rapidly as the policy approaches the K\&K Logic reward ceiling, because further updates no longer convert into measurable reward gains. The multiplicative fusion in Eq.~\eqref{eq:multiplicative_fusion} couples the two signals, so \textbf{the collapse of progress pulls the fused utility down even while learnability remains nominally healthy}, and the controller redirects rollouts toward arms where additional updates still translate into reward. This is the failure mode that motivates the conversion factor $\ConvertFactor{i}$ in Section~\ref{sec:reward_construction}, and Figure~\ref{fig:pac_group2_7b_heatmaps} shows that \pac corrects it online, while training is still in progress, rather than only in retrospect.

\subsection{A Gradual Rather Than Abrupt Transition}

The Math and Code arms gain budget through the same factorization but with a different signal mix. For Math, both learnability and progress decay only slowly throughout training; for Code, the progress signal strengthens once its rollouts reliably pass unit-test verification. Both arms then offer a more favorable fused utility than the saturating K\&K Logic arm, and Thompson Sampling raises their posterior means accordingly. The transition is gradual rather than abrupt because \textbf{the variance inflation in Eq.~\eqref{eq:variance_inflation} preserves posterior uncertainty and prevents premature exploitation} of the arm with the largest current utility estimate. The same mechanism explains why the K\&K Logic share stabilizes near a quarter of the batch instead of collapsing to zero: with inflated variance, the controller keeps K\&K Logic available for periodic resampling, so that it can still detect any later utility shift on that arm.

\subsection{From Reallocation to the Multi-Domain Validation Gain}

The mid-training reallocation provides a mechanistic account of the multi-domain trends in Section~\ref{sec:results_discussion}. The \pac validation curve in Figure~\ref{fig:main_results_multi_domain} begins to separate from DUMP after the mid-training phase, and this separation coincides with the window in which the controller redirects rollout budget away from the saturating K\&K Logic arm and toward Math and Code. Reward-derived progress is therefore the decisive signal in the regime where advantage magnitude alone would mislead the controller, and \textbf{the resulting reallocation accounts for the 6.2\% cross-domain average improvement of \pac over advantage-only baselines} reported for Qwen2.5-7B in Table~\ref{tab:multi_domain_validation}.

\raggedbottom
\section{Prompt Examples}
\label{app:prompt_examples}

This appendix shows representative prompts from the datasets used in the two experimental settings.
To keep the examples readable, we include the task-specific user content and omit the shared conversation preamble and assistant cue.
We also report the reward target for each example to clarify how automatic verification is performed; these targets are not part of the model input.

\subsection{Setting 1: Multi-Level Reasoning}

\begin{promptbox}{Countdown Prompt}
\small
\textbf{Prompt:}

Using the numbers \(21, 95, 8, 2, 26, 2\) and basic arithmetic operations \((+, -, \times, /)\), create an expression that equals \(30\).
You need to use each number exactly once.
Present the final expression that equals the target number within \texttt{\textbackslash boxed\{\}}.
For example: \texttt{\textbackslash boxed\{46+68/(52-50)\}}.

\textbf{Reward target:} The expression must use \(21, 95, 8, 2, 26, 2\) exactly once and evaluate to \(30\).
\end{promptbox}

\begin{promptbox}{Zebra Prompt}
\small
\textbf{Prompt:}

This is a logic puzzle.
There are \(4\) houses, numbered \(1\) on the left to \(4\) on the right, from the perspective of someone standing across the street from them.
Each house has a different person, and each person has different characteristics:
\begin{itemize}
    \item Each person has a unique name: carol, arnold, alice, bob.
    \item Everyone has a favorite smoothie: butterscotch, desert, darkness, dragonfruit.
    \item They all have a different favorite flower: iris, daffodils, lilies, carnations.
    \item Everyone has something different for lunch: grilled cheese, stir fry, pizza, soup.
\end{itemize}

\begin{enumerate}
    \item Bob is the person who loves stir fry.
    \item The person who loves eating grilled cheese is directly left of the person who loves a carnations arrangement.
    \item The person who loves a carnations arrangement and the person who loves a bouquet of daffodils are next to each other.
    \item The person who loves the soup is directly left of the Butterscotch smoothie drinker.
    \item The Darkness smoothie drinker is directly left of the Desert smoothie lover.
    \item The person who loves the soup is the Desert smoothie lover.
    \item The person who loves a carnations arrangement is Bob.
    \item The person who loves the bouquet of lilies is the Butterscotch smoothie drinker.
    \item The Desert smoothie lover is Carol.
    \item The person who loves the bouquet of iris is Arnold.
\end{enumerate}

What is the name of the person who lives in House \(1\)?
Provide only the name as the final answer and put it in \texttt{\textbackslash boxed\{\}}.

\textbf{Reward target:} \texttt{arnold}.
\end{promptbox}

\begin{promptbox}{ARC Prompt}
\small
\textbf{Prompt:}

Find the common rule that maps an input grid to an output grid, given the examples below.

\textbf{Example 1:}\\
Input: \texttt{0 0 0 9 3 6 0 0 0 0}\\
Output: \texttt{9 3 6 0 0 0 0 0 0 0}

\textbf{Example 2:}\\
Input: \texttt{0 0 0 0 0 0 0 0 2 8}\\
Output: \texttt{0 0 0 0 0 2 8 0 0 0}

\textbf{Example 3:}\\
Input: \texttt{0 0 0 0 0 0 0 9 0 0}\\
Output: \texttt{0 0 0 0 9 0 0 0 0 0}

Below is a test input grid.
Predict the corresponding output grid by applying the rule you found.
Describe how you derived the rule and your overall reasoning process before submitting the answer.
The final answer must be placed in \texttt{\textbackslash boxed\{\}} and should be just the test output grid itself.

\textbf{Input:}\\
\texttt{1 6 7 8 9 4 0 0 0 0}

\textbf{Reward target:} \texttt{8 9 4 0 0 0 0 1 6 7}.
\end{promptbox}

\subsection{Setting 2: Multi-Domain Reasoning}

\begin{promptbox}{Math Prompt}
\small
\textbf{Prompt:}

Every morning Aya goes for a \(9\)-kilometer-long walk and stops at a coffee shop afterwards.
When she walks at a constant speed of \(s\) kilometers per hour, the walk takes her \(4\) hours, including \(t\) minutes spent in the coffee shop.
When she walks \(s+2\) kilometers per hour, the walk takes her \(2\) hours and \(24\) minutes, including \(t\) minutes spent in the coffee shop.
Suppose Aya walks at \(s+\frac{1}{2}\) kilometers per hour.
Find the number of minutes the walk takes her, including the \(t\) minutes spent in the coffee shop.
Put the final answer within \texttt{\textbackslash boxed\{\}}.

\textbf{Reward target:} \(\boxed{204}\).
\end{promptbox}

\begin{promptbox}{Code Prompt}
\small
\textbf{Prompt:}

You are a helpful programming assistant.
The reasoning process and answer are enclosed within \texttt{<think>...</think>} and \texttt{<answer>...</answer>} tags, respectively.

\textbf{User:}
Generate a random string of the specified length composed of uppercase and lowercase letters, and then count the occurrence of each character in this string.
The function should raise \texttt{ValueError} if the length is a negative number.
The function should output:

\begin{quote}
\small
\texttt{dict}: a dictionary where each key is a character from the generated string and the value is the count of how many times that character appears in the string.
\end{quote}

You should write self-contained code starting with:
\begin{quote}
\small\ttfamily
import collections\\
import random\\
import string\\
def task\_func(length=100):
\end{quote}

\textbf{Reward target:} Unit tests validate the implementation of entry point \texttt{task\_func}.
\end{promptbox}

\begin{promptbox}{K\&K Logic Prompt}
\small
\textbf{Prompt:}

A very special island is inhabited only by knights and knaves.
Knights always tell the truth, and knaves always lie.
You meet \(6\) inhabitants: Grace, Penelope, Mia, Noah, Matthew, and Lily.
Grace said, ``Noah is a knave.''
Penelope asserted, ``Mia is a knight.''
Mia said, ``Grace is a knave.''
Noah said, ``Mia is a knight if and only if Matthew is a knight.''
Matthew said that Noah is not a knave.
Lily said, ``Mia is not a knight.''
So who is a knight and who is a knave?

\textbf{Reward target:}
\begin{quote}
\small
(1) Grace is a knave\\
(2) Penelope is a knight\\
(3) Mia is a knight\\
(4) Noah is a knight\\
(5) Matthew is a knight\\
(6) Lily is a knave
\end{quote}
\end{promptbox}

\end{document}